\documentclass[runningheads]{llncs}
\usepackage[paperheight=235mm,
  paperwidth=155mm,textwidth=12.2cm,textheight=19.3cm]{geometry}

\usepackage{bbding}

\usepackage[utf8]{inputenc} 
\usepackage[T1]{fontenc}    
\usepackage{url}            
\usepackage{nicefrac}       
\usepackage{microtype}      
\usepackage[dvipsnames]{xcolor}
\usepackage[normalem]{ulem}
\usepackage{graphicx}

\usepackage{amsmath,amsfonts,amssymb,amstext,mathtools,bm,amsthm}
\usepackage{booktabs}
\usepackage{outlines}
\usepackage{subcaption}
\usepackage{tablefootnote}
\usepackage{wrapfig}
\usepackage{paralist} 

\usepackage{macros}

\def\noeditingmarks{}
\usepackage{peanutgallery}

\makeatletter
\RequirePackage[bookmarks,unicode,colorlinks=true]{hyperref}%
   \def\@citecolor{blue}%
   \def\@urlcolor{blue}%
   \def\@linkcolor{blue}%
   
\renewcommand\paragraph{\@startsection{paragraph}{4}{\z@}%
                       {-5\p@ \@plus -4\p@ \@minus -4\p@}%
                       {-0.5em \@plus -0.22em \@minus -0.1em}%
                       {\normalfont\normalsize\bf}}
\def\orcidID#1{\smash{\href{http://orcid.org/#1}{\protect\raisebox{-1.25pt}{\protect\includegraphics{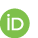}}}}}
\makeatother

\begin{document}

\title{Efficient Neural Network Analysis with Sum-of-Infeasibilities}

\author{
  Haoze Wu\inst{1}\href{mailto:haozewu@stanford.edu}{\Envelope}\orcidID{0000-0002-5077-144X}  \and
  Aleksandar Zelji\'c\inst{1}\orcidID{0000-0003-0673-9327} \and
  Guy Katz\inst{2}\orcidID{0000-0001-5292-801X} \and 
  Clark Barrett\inst{1}\orcidID{0000-0002-9522-3084}
}

\authorrunning{H. Wu et al.}

\institute{
  Stanford University, Stanford, USA \\ \and
  The Hebrew University of Jerusalem, Jerusalem, Israel \\
}

\maketitle
\begin{abstract}
  Inspired by sum-of-infeasibilities methods in convex optimization,
  we propose a novel procedure for analyzing verification
  queries on neural networks with piecewise-linear activation functions.
  Given a convex relaxation which over-approximates the non-convex activation functions,
  we encode the violations of activation functions as a cost function
  and optimize it with respect to the convex relaxation.
  The cost function, referred to as the Sum-of-Infeasibilities (\soi),
  is designed so that its minimum is zero and achieved only if all the
  activation functions are satisfied. We propose a stochastic procedure,
  \sys, to efficiently minimize the \soi. An extension to a canonical case-analysis-based
  complete search procedure can be achieved by replacing the convex procedure executed
  at each search state with \sys. Extending the complete search with \sys achieves multiple simultaneous
  goals: 1) it guides the search towards a counter-example; 2) it enables more informed
  branching decisions; and 3) it creates additional opportunities for bound derivation.
  An extensive evaluation across different benchmarks and solvers demonstrates the benefit
  of the proposed techniques. In particular,
  we demonstrate that \soi significantly improves the performance of
  an existing complete search procedure.  Moreover, the \soi-based implementation outperforms
  other state-of-the-art complete verifiers.  We also show that our technique
  can efficiently improve upon the perturbation bound derived by a recent adversarial attack algorithm.

\end{abstract}

\section{Introduction}
\label{sec:intro}

Neural networks have become state-of-the-art solutions in various application
domains, e.g., face recognition, voice
recognition, game-playing, and automated piloting~
\cite{masi2018deep,hinton2012deep,silver2016mastering,selfDriving}.
While generally successful, neural networks are known to be susceptible
  to input perturbations that humans are naturally invariant to~\cite{lbfgs,DBLP:conf/iclr/KurakinGB17a}.
This calls the trustworthiness of neural networks into
question, particularly in safety-critical domains.

In recent years, there has been a growing interest in applying formal methods to
neural networks to analyze certain robustness or safety \emph{specifications}~\cite{liu2019algorithms}.
Such specifications are often defined by a collection of partial
input/output relations: e.g., the network uniformly and correctly
classifies inputs within a certain distance (in some $l_p$ norm) of a selection
of input points. The goal of formal verification is to either prove that the
network meets the specification or to disprove it by constructing
a counter-example. 





Most standard activation functions in neural networks are non-linear,
making them challenging to reason about.
Consider the rectified linear unit (ReLU): if a ReLU can take both positive and negative inputs,
a verifier will typically need to consider, separately, each of these
two activation phases. Naive case analysis requires exploring a number of
combinations that is exponential in the number of ReLUs, which
quickly becomes
computationally infeasible for large networks.
To mitigate this complexity, neural network
verifiers typically operate on convex relaxations of the activation functions.
The relaxed problem can often be solved with an efficient convex procedure,
such as Simplex~\cite{reluplex,planet} or (sub-)gradient methods~\cite{sdp,dual1}.  
Due to the relaxation, however, a solution may be inconsistent with the true activation
functions.  When this happens, the convex procedure cannot make further
progress on its own. For this reason, to ensure completeness, the convex procedure
is typically embedded in an exhaustive search shell,
which encodes the activation functions explicitly and branches on them when needed.
While the exhaustive search ensures progress, it also brings back the problem of
combinatorial explosion.
This raises the key question:{ \bf can we guide the convex procedure to satisfy the
  activation functions without explicitly encoding them?} 


In convex optimization, the sum-of-infeasibilities (\soi) \cite{boyd2004convex}
function measures the error (with respect to variable bounds) of a variable
assignment. Minimizing the \soi naturally guides the procedure to a
satisfying assignment. In this paper, we extend this idea to instead
represent the error in the non-linear activation functions. The goal is to
``softly'' guide the search over the relaxed problem using information
about the precise activation functions. If an assignment is found for which
the \soi is zero, then not only is the assignment a solution for the relaxation,
but it also solves the precise problem.
%
%
Encoding the \soi w.r.t. the piecewise-linear activation functions yields a
concave piecewise-linear function, which is challenging to minimize directly.
Instead, we propose to minimize the \soi for individual \emph{activation patterns}
and reduce the \soi minimization to a \emph{stochastic} search for
the activation pattern where the \soi is minimal.
The advantage is that for each activation pattern, the \soi collapses
into a \emph{linear} cost function, which can be easily handled by a convex solver.
We introduce a specialized procedure, \sys, which uses Markov chain Monte Carlo
(MCMC) search to efficiently navigate towards activation patterns at the global
minimum of the \soi. If the minimal \soi is ever zero for an activation pattern,
then a solution has been found.

An extension to a canonical complete search procedure can be achieved by
replacing the convex procedure call at each search state with the \sys procedure.
Since the \soi contains additional information about the problem,
we propose a novel \soi-aware branching heuristic based on the estimated
impact of each activation function on the \soi.
Finally, \sys naturally preserves new bounds derived during the execution of
the underlying convex procedure (e.g., Simplex), which further prunes the search space
in the complete search.
For simplicity, we focus on ReLU activation functions in this paper,
though the proposed approach can be applied to any piecewise-linear activation function.


We implemented the proposed techniques in the Marabou framework for Neural
Network Analysis~\cite{marabou} and performed an extensive performance evaluation on a wide
range of benchmarks. We compare against multiple baselines and
show that extending a complete search procedure with
our \soi-based techniques results in significant overall speed-ups.
Finally, we present an interesting use case for our
procedure --- efficiently improving the perturbation bounds
found by AutoAttack~\cite{autoattack}, a state-of-the-art adversarial attack algorithm.


To summarize, the contributions of the paper are:
\begin{inparaenum}[(i)]
\item  a technique for guiding a convex solver with an \soi function w.r.t. the
  activation functions;
\item  \sys --- a procedure for minimizing the non-linear SoI via the interleaving
  use of an MCMC sampler and a convex solver;
\item an \soi-aware branching heuristic, which complements the integration of \sys
  into a case-analysis based search shell; and
\item a thorough evaluation of the proposed techniques. 
\end{inparaenum}

The rest of the paper is organized as follows. Section~\ref{sec:related} presents an
overview of related work. Section~\ref{sec:prelim} introduces preliminaries.
Section~\ref{sec:soi} introduces the \soi and proposes a solution for its minimization.
Section~\ref{sec:complete} presents the analysis procedure 
\sys, its use in the complete verification setting, and an \soi-aware branching heuristic.
Section \ref{sec:experiments} presents an extensive experimental evaluation. Conclusions
and future work are in Section \ref{sec:concl}.

\section{Related Work}
\label{sec:related}
Approaches to complete analysis of neural networks can be divided into
SMT-based~\cite{reluplex,marabou,planet}, reachability-analysis
based~\cite{nnenum,nnv,rpm,deepsplit,fromherz2020fast}, and the more general branch-and-bound
approaches~\cite{optAndAbs,mipverify,mip01,branching,babsr,peregrinn,dependency}.
As mentioned in \cite{unified}, these approaches are related, and differ
primarily in their techniques for bounding and branching.
Given the computational complexity of neural network verification, a diverse set
of research directions aims to improve performance in practice. Many approaches
prune the search space using tighter convex relaxations and bound
inference techniques~\cite{nnv,planet,dlv,deeppoly,kpoly,frown,crown,wong,reluval,neurify,sherlock,mipverify,fastlin,barrier,barrier-revisited,sdp,star,AI2,bcrown,singh2019boosting,cnn-cert,deepz}. Another direction leverages
parallelism by exploiting independent structures in the search space
\cite{gpupoly,xu2020fast,wu2020parallelization}.
Different encodings of the neural network verification problems have also been
studied: e.g., as MILP problems that can be tackled by off-the-shelf solvers
~\cite{mipverify,formalism}, or as dual problems admitting efficient GPU-based
algorithms~\cite{dual,dual1,dual2,scaling}.
\sys can be instantiated with any sound convex relaxations and matching convex
procedures. It can also be installed in any case-analysis-based
complete search shell, therefore integrating easily with existing parallelization
techniques, bound-tightening passes, and branching heuristics.



Two approaches most relevant to our work are Reluplex~\cite{reluplex} and
\peregrinn~\cite{peregrinn}. Reluplex invokes an LP solver to solve the relaxed
problem, and then updates its solution to satisfy the violated activation
functions --- with the hope of nudging the produced solutions
towards a satisfying assignment. 
However, the updated solution by Reluplex could violate the linear relaxation,
leading to non-convergent cycling between solution updates and LP solver
calls, which can only be broken by branching.
In contrast, our approach uses information
about the precise activation functions to \emph{actively} guide the convex solver.
Furthermore, in the limit \sys converges to a solution (if one exists).
\peregrinn
also uses an objective function to guide the solving of the convex relaxation.
However, their objective function \emph{approximates} the ReLU violation and does
not guarantee a real counter-example when the minimum is reached. In contrast, the \soi function
captures the \emph{exact} ReLU violation, and if a zero-valued point is found,
it is \emph{guaranteed} to be a real counter-example. We compare our techniques to \peregrinn
in Section~\ref{sec:experiments}.

We use MCMC-sampling combined with a convex procedure to minimize the concave
piecewise-linear \soi function. MCMC-sampling is a common approach for
stochastically minimizing irregular cost functions that are not amenable to exact
optimization techniques~\cite{jia2018beyond,schkufza2013stochastic,andrieu2003introduction}.
Other stochastic local search techniques~\cite{selman1994noise,gent1995hybrid}
could also be used for this task. However, we chose MCMC because it is adept at
escaping local optima, and in the limit, it samples more frequently the region
around the optimum value.
As one point of comparison, in Section~\ref{sec:experiments}, we compare MCMC-sampling
with a Walksat-based~\cite{selman1994noise} local search strategy.

\section{Preliminaries}
\label{sec:prelim}

\xhdr{Neural Networks}
We define a feed-forward, convolutional, or residual neural network with $k+1$ layers
as a set of \emph{neurons} $N$, topologically ordered into
\emph{layers} $L_0,..., L_k$, where $L_0$ is the input layer and $L_k$ is the output layer.
Given $n_i, n_j \in N$, we use $n_i \prec n_j$ to denote that the layer of $n_i$ precedes
the layer of $n_j$. The value of a neuron
$n_i \in N \backslash L_0$ is computed as $\act_i(b_i + \sum_{n_j \prec n_i}
w_{ij} * n_j)$,
an affine transformation of the preceding neurons followed by an activation function $\act_i$. We use $n^b_i$ and $n^a_i$ to represent
the pre- and post-activation values of such a neuron: $n^a_i = \act_i(n^b_i)$.  For $n_i\in L_0$, $n^b_i$ is undefined and we assume $n^a_i$ can take any value.
In this paper, we focus on ReLU neural networks. That is, $\act_i$
is the ReLU function ($\mathit{ReLU}(x) = \max(0,x)$) unless $n_i$ belongs to the output
layer $L_k$, in which case $\act_i$ is the identity function.  We use $\relu(N)$ to denote the set of ReLU neurons in $N$.  An
\emph{activation pattern} is defined by choosing a particular phase (either active or inactive) for every $n\in\relu(N)$ (i.e., choosing either $n^b_i < 0$ or $n^b_i \ge 0$ for each $n_i \in\relu(N)$).

\xhdr{Neural Network Verification as Satisfiability}
Consider the verification of a property $P$ over a neural network
$N$. The property $P$ has the form $P_{in} \Rightarrow P_{out}$, where
$P_{in}$ and $P_{out}$ constrain the input and output layers, respectively.
$P$ states that for each input point satisfying $P_{in}$, the output layer
satisfies $P_{out}$.
To formalize the verification problem, we first define
the set of \emph{variables} in a neural network $N$, denoted as $\var(N)$, to be $\cup_{n_i \in N\backslash L_k} \{n^a_i\}\ \cup\ \cup_{n_i\in N\backslash L_0} \{n^b_i\}$.
We
define a \emph{variable assignment}, $\alpha: \var(N) \rightarrow \mathbb{R}$,
to be a mapping from variables in $N$ to real values.
The verification task thus can be formally stated as finding a variable
assignment $\alpha$ that satisfies the following set of \emph{constraints}
over $\var(N)$ (denoted as $\phi$):%
\footnote{The verification can also be equivalently viewed as an optimization
problem~\cite{unified}.}
\begin{subequations}
  \begin{gather}
    \forall n_i \in
    N\backslash L_0, n^b_i = b_i + \sum_{n_j \prec n_i} w_{ij} * n^a_j \label{eq:neuron} \\
    \forall n_i \in \relu(N), n^a_i = \act_i (n^b_i) \label{eq:act}\\
    P_{in} \land \neg P_{out} \label{eq:prop}
  \end{gather}
  \label{eq:exact}
\end{subequations}

\vspace*{-3ex}
If such an assignment $\alpha$ exists, we say that $\phi$ is \emph{satisfiable} and
can conclude that $P$ does not hold, as from $\alpha$ we can retrieve an
input $x \in P_{in}$, such that the neural network's output violates $P_{out}$.
If such an $\alpha$ does \emph{not} exist, we say $\phi$ is \emph{unsatisfiable} and can
conclude that $P$ holds.
We use $\alpha \models \phi$ to denote
that an assignment $\alpha$ satisfies $\phi$. In short, verifying whether $P$
holds on a neural network $N$ boils down to deciding the satisfiability of
$\phi$. We refer to $\phi$ also as the \emph{verification query}
in this paper.

\xhdr{Convex Relaxation of Neural Networks} Deciding whether $P$ holds on a ReLU
network $N$ is NP-complete~\cite{reluplex}.
To curb intractability, many verifiers consider the convex (e.g., linear, semi-definite)
relaxation of the verification problem, sacrificing completeness in exchange
for a reduction in the computational complexity.
We use $\rlx{\phi}$ to denote the convex relaxation of the exact problem $\phi$.
If $\rlx{\phi}$ is unsatisfiable, then $\phi$ is also unsatisfiable, and property $P$ holds.
If the convex relaxation is satisfiable with satisfying assignment $\alpha$ and
$\alpha$ also satisfies $\phi$, then $P$ does not hold.

\begin{wrapfigure}{r}{0.45\linewidth}
  \vspace{-9mm}
  \centering
  \begin{tikzpicture}[scale=0.85]
    \draw[->, very thick] (-3, 0) -- (3, 0) node[right] {$n^b$};
    \draw[->, very thick] (0, 0) -- (0, 3) node[above] {$n^a$};
    \draw[domain=-2.5:0, very thick, variable=\x, blue] plot ({\x}, {0});
    \draw[domain=0:2.5, very thick, variable=\x, blue] plot ({\x}, {\x});
    \draw[domain=-2.5:2.5, very thick, dashed, variable=\x, blue] plot ({\x}, {0.5*\x+1.25});

    \draw[domain=0:2.5, thick, dashed, variable=\y] plot ({2.5}, {\y});

    \path[fill=red, opacity=0.25] (-2.5,0) -- (0,0) -- (2.5,2.5) -- cycle;

    \node at (0, -0.3) {0};
    \node at (2.5, -0.3) {$u$};
    \node at (-2.5, -0.3) {$l$};
  \end{tikzpicture}
  \vspace{-7mm}
  \caption{The Planet relaxation.}
  \label{fig:relu}
    \vspace{-7mm}
\end{wrapfigure}
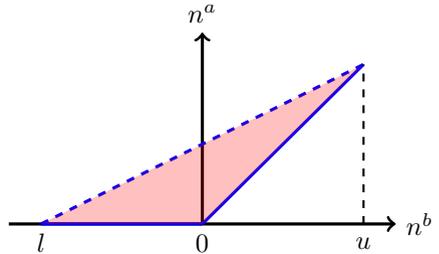

In this paper, we use the \emph{Planet relaxation} introduced in \cite{planet}.
It is a linear relaxation, illustrated in Figure \ref{fig:relu}.
Each ReLU constraint
ReLU($n^b$) = $n^a$ is over-approximated by three linear constraints: $n^a \geq 0$,
$n^a \geq n^b$, and $n^a \leq \frac{u}{u - l} n^b - \frac{u * l}{u - l}$, where
$u$ and $l$ are the upper and lower bounds of $n^b$,
respectively (which can be derived using bound-tightening techniques such as those in~\cite{reluval,deeppoly,crown}).
If Constraint \ref{eq:prop} is also linear, the convex relaxation $\rlx{\phi}$ is
a Linear Program, whose satisfiability can be decided efficiently (e.g., using
the \emph{Simplex} algorithm~\cite{simplex}).

\xhdr{Sum-of-Infeasibilities} In convex optimization~\cite{boyd2004convex,king2013simplex},
the sum-of-infeasibilities (SOI) method can be used to direct the feasibility search.
The satisfiability of a formula $\phi$ is cast as an optimization problem, with
an objective function representing the total \emph{error} (i.e., the sum of
the distances from each out-of-bounds variable to its closest bound).
The lower bound of $f$ is 0
and is achieved only if $\phi$ is satisfiable.
In our context, we use a similar function $\soifunc$, but with the
difference that it represents the total error of the ReLU constraints in
$\phi$. In our case, $\soifunc$ is non-convex, and thus a more
sophisticated approach is needed to minimize it efficiently.

\xhdr{Complete Analysis via Exhaustive Search}
One common approach for complete
verification involves constructing a search tree and calling a \emph{convex
procedure}   $\checkSat$  at each tree node,
as shown in Algorithm~\ref{alg:complete}. $\checkSat$ solves the convex
relaxation
$\rlx{\phi}$ and returns a pair $r,\alpha$ where either:
1) $r=\sat{}$ and $\alpha\models\rlx{\phi}$; or
2) $r=\unsat{}$ and $\rlx{\phi}$ is unsatisfiable.
If $\rlx{\phi}$ is unsatisfiable or $\alpha$ also satisfies $\phi$, then
the result for $\rlx{\phi}$ also holds for $\phi$ and is returned.
Otherwise, the search space
is divided further using $\branch$, which returns a set $\Psi$ of sub-problems such
that $\phi$ and $\bigvee \Psi$ are equisatisfiable.

\begin{wrapfigure}{r}{0.53\textwidth}
  \vspace{-1.6cm}
  \begin{minipage}{6.5cm}
    \begin{algorithm}[H]
      \small
      \begin{algorithmic}[1]
        \State {\bfseries Input:} a verification query $\phi$.
        \State {\bfseries Output:} \sat{}/\unsat{}
        \Function{completeSearch}{$\phi$}
        \State $\phi \assign \tightenBounds(\phi)$
        \State {$r, \alpha \assign \checkSat(\rlx{\phi}$)}
        \If{$r = \unsat{} \lor \alpha\models\phi$}
        \State {\bf return} $r$
        \EndIf
        \For {$\phi_i \in \branch(\phi)$}
        \If {$\textsc{completeSearch}(\phi_i) = \sat{}$}
        \State {\bf return} \sat{}
        \EndIf
        \EndFor
        \State {\bf return} \unsat{}
        \EndFunction
      \end{algorithmic}
      \caption{Complete search.\label{alg:complete}}
    \end{algorithm}
  \end{minipage}
  \vspace{-0.9cm}
\end{wrapfigure}
Before invoking $\checkSat$ to solve $\rlx{\phi}$, it is common to first
call an efficient bound-tightening procedure ($\tightenBounds$) to prune the
search space or even derive \unsat{} preemptively.  This
$\tightenBounds$ procedure can be instantiated in various ways, including with analyses
based on LiPRA~\cite{lipra,crown,deeppoly,wong}, kReLU~\cite{kpoly}, or
PRIMA~\cite{prima}. In addition to the dedicated bound-tightening
pass, some convex procedures (e.g., Simplex) also naturally lend themselves to
bound inference during their executions \cite{kingThesis,reluplex}.
The overall performance of Algorithm \ref{alg:complete} depends on
the efficacy of bound-tightening, the branching heuristics,
and the underlying convex procedure.


\xhdr{Adversarial attacks}
Adversarial attacks \cite{lbfgs,pgd,fgsm,cw} are another approach
for assessing neural network robustness.  While verification uses exhaustive
search to either prove or disprove a particular specification, adversarial
attacks focus on efficient heuristic algorithms for the latter. From another
perspective, they can demonstrate \emph{upper bounds} on neural network robustness.
In Section \ref{sec:experiments}, we show that our analysis procedure can
improve the bounds found by AutoAttack~\cite{autoattack}.

\section{Sum of Infeasibilities in Neural Network Analysis}
\label{sec:soi}
In this section, we introduce our \soi function,
consider the challenge of its minimization, and present a
stochastic local search solution.

\subsection{The Sum of Infeasibilities}

As mentioned above, in convex optimization, an \soi function represents the sum
of errors in a candidate variable assignment.  Here, we build on this
idea by introducing a cost function $\soifunc$, which computes the sum of errors
introduced by a convex relaxation of a verification query.
%
We aim to use $\soifunc$ to reduce the satisfiability problem for $\phi$ to a
simpler optimization problem.  We will need the following property to hold.
\begin{condition}
  \label{cond:soi}
  For an assignment $\alpha$, $\alpha \models \phi$ iff
  $\alpha \models \rlx{\phi} \wedge \soifunc \leq 0$. \label{condition:soi}
\end{condition}
If Condition~\ref{cond:soi} is met, then satisfiability of $\phi$ reduces to
the following minimization problem:
\begin{equation}
  \begin{aligned}
    \minimize_{\alpha} \quad &\soifunc\\
    \textrm{subject to} \quad &\alpha\models\rlx{\phi}
  \end{aligned}
  \label{eq:min}
\end{equation}

To formulate the \soi for ReLU networks, we first define the error in a ReLU
constraint $n$ as:
\begin{equation}
\vio(n) = \min(n^a - n^b, n^a)
\label{eq:vio}
\end{equation}
The two arguments correspond to the error when the ReLU is in the active and inactive
phase,
respectively. Recall that the Planet relaxation constrains $(n^b, n^a)$ in
the triangular area in Figure \ref{fig:relu}, where $n^a \geq n^b$ and $n^a \geq
0$. Thus, the minimum of $\vio(n)$ subject to $\rlx{\phi}$ is non-negative, and
furthermore, $\vio(n) = 0$ iff the ReLU constraint $n$ is satisfied (this is
also true for any relaxation at least as tight as the Planet relaxation).
%
We now define $\soifunc$ as the sum of errors in individual ReLUs:
\begin{equation}
  \soifunc = \sum_{n \in \relu(N)}\vio(n)
  \label{eq:soi}
\end{equation}
\begin{theorem}
  Let $N$ be a set of neurons for a neural network, $\phi$ a verification query
  (an instance of~\eqref{eq:exact}), and $\rlx{\phi}$ the planet relaxation of
  $\phi$.
  Then $\soifunc$ as given by~\eqref{eq:soi} satisfies Condition \ref{cond:soi}.
\end{theorem}
\vspace{-3mm}
\begin{proof}
  It is straightforward to show that $\soifunc$ subject to $\rlx{\phi}$ is
  non-negative and is zero if and only if each $\vio(n_i)$ is zero. That is,
  $\min{\soifunc}$ subject to $\rlx{\phi}$ is zero if and only if all ReLUs are
  satisfied. Therefore, if $\alpha$ satisfies $\phi$, then $\alpha \models \soifunc = 0$. On the
  other hand, since an assignment $\alpha$ that satisfies $\rlx{\phi}$ can only
  violate the ReLU constraints in $\phi$, if $\alpha \models \soifunc = 0$,
  then all the constraints in $\phi$ must be satisfied, i.e., $\alpha \models \phi$.
\end{proof}
Note that the error $\vio$, and its extension to \soi, can easily be defined for other
piecewise-linear functions besides ReLU.
%
%
We now turn to the question of minimizing $\soifunc$.  Observe that
\begin{equation}
  \min\soifunc = \min \sum_{n\in\relu(N)}\vio(n) = \min\Big(\big\{ f \mid
  f = \sum_{n_i \in \relu(N)} t_i, \quad t_i \in \{n^a_i - n^b_i, n^a_i\} \big\}\Big).
  \label{eq:rearrange}
\end{equation}
Thus, $\soifunc$ is the minimum over a set, which we will denote $\soiset$, of
linear functions. Although $\min \soifunc$
cannot be used directly as an objective in a convex procedure, we could
minimize each individual linear function $f\in \soiset$ with a convex procedure
and then keep the minimum over all functions.
We refer to the functions in $\soiset$ as \emph{phase patterns} of $\soifunc$.
For notational convenience, we define $\costOf(f,\phi)$ to be the minimum of $f$ subject to
$\phi$.
The minimization problem \eqref{eq:min} can thus be
restated as searching for the phase pattern $f\in \soiset$, where $\costOf(f,\rlx{\phi})$ is
minimal.
Note that for a particular \emph{activation pattern}, $\soifunc = f$ for
some $f\in\soiset$.
From
this perspective, searching for the $f\in \soiset$ where $\costOf(f,\rlx{\phi})$ is
minimal can also be viewed as searching for the activation pattern where the
global minimum of $\soifunc$ is achieved.

\subsection{Stochastically Minimizing the SoI with MCMC Sampling \label{subsec:minSOI}}
In the worst case, finding the minimal value of $\costOf(f,\rlx{\phi})$ requires enumerating and minimizing
each $f$ in $\soiset$ (or equivalently, minimizing $\soifunc$ for each activation pattern),
which has size $2^{\abs{\relu(N)}}$. However, importantly, the search can terminate
immediately if a phase pattern $f$ is found such that $\costOf(f,\rlx{\phi}) = 0$.
We leverage this fact below.  Note that each phase pattern has $\abs{\relu(N)}$ 
adjacent phase patterns, each differing in only one linear subexpression.
The space of phase patterns is thus fairly dense, making it amenable to traversal
using stochastic local search methods.
%
In particular, intelligent hill-climbing algorithms, which can be made
robust against local optima, are well suited for this task.

Markov chain Monte Carlo (MCMC)~\cite{mcmc-1} methods are such an approach.
In our context, MCMC methods can be used to generate a
sequence of phase patterns $f_0, f_1, f_2... \in \soiset$, with the desirable 
property that in the limit, the phase patterns are more frequently
from the minimum region of $\costOf(f,\rlx{\phi})$.

We use the Metropolis-Hastings (M-H) algorithm~\cite{mh}, a widely
applicable MCMC method, to construct the sequence.
The algorithm maintains a current phase pattern $f$ and proposes
to replace $f$ with a new phase pattern $f'$. The proposal
comes from a \emph{proposal distribution} $q(f' | f)$ and
is accepted with a certain \emph{acceptance probability} $m(f{\rightarrow}f')$.
If the proposal is accepted, $f'$ becomes the new current phase pattern. Otherwise,
another proposal is considered. This process is repeated until one of the following
scenarios happen:
1) a phase pattern $f$ is chosen with $\costOf(f,\rlx{\phi}) = 0$;
2) a predetermined computational budget is exhausted; or
3) all possible phase patterns have been considered.
The last scenario is generally infeasible for non-trivial networks. 
%
In order to employ the algorithm, we transform $\costOf(f,\rlx{\phi})$ into a probability
distribution $p(f)$ using a common method~\cite{mcmc}:
\[p(f) \propto \exp(-\beta \cdot \costOf(f,\rlx{\phi}))\]
where $\beta$ is a configurable parameter. If the proposal distribution is
symmetric
(i.e., $q(f| f') = q(f'| f)$), the acceptance probability is the following
(often referred to as the \emph{Metropolis ratio})~\cite{mcmc}:
\[
m(f{\rightarrow}f') = \min(1, \frac{p(f')}{p(f)}) = \min\bigg(1, \exp\Big(-\beta \cdot \big(\costOf(f',\rlx{\phi}) - \costOf(f,\rlx{\phi})\big)\Big)\bigg)
\]
Importantly, under this acceptance probability, \textit{a proposal reducing the
  value of the cost function is always accepted,
  while a proposal that does not may still be accepted}
(albeit with a probability that is inversely correlated with the increase in the cost).
This means that the algorithm always greedily moves to a lower cost phase
pattern whenever it can, but it also has an effective means for escaping local minima.
Note that since the sample space is finite, as long as the proposal strategy is \emph{ergodic},%
\footnote{A proposal strategy is ergodic if it is
capable of transforming any phase pattern to any other through a sequence of applications.
We use a symmetric and ergodic proposal distribution as explained in Section~\ref{subsec:deepsoi}.}
in the limit, the probability of sampling \emph{every} phase pattern
(therefore deciding the satisfiability of $\phi$) converges to 1.
However, we do not have formal guarantees about the convergence rate, and it is usually
impractical to prove unsatisfiability this way. Instead, as we shall see in the next section,
we enable complete verification by embedding the M-H algorithm in an exhaustive search shell.
\begin{outline}

\end{outline}

\section{The \sysT Algorithm} 
\label{sec:complete}

In this section, we introduce \sys, a novel verification algorithm
that leverages the \soi function, and show
how to integrate it with a complete verification procedure.  We also discuss the impact of \sys on
complete verification and propose an \soi-aware branching heuristic.

\subsection{\sysT
\label{subsec:deepsoi}}

\begin{wrapfigure}{r}{0.66\textwidth}
\vspace{-1.5cm}
\begin{minipage}[t]{7.65cm}
\begin{algorithm}[H]
\small
\begin{algorithmic}[1]
  \State {\bfseries Input:} A verification query $\phi$.
  \State {\bfseries Output:} \sat{}/\unsat{}/\unknown{}
  \Function{\sysT}{$\phi$}
  \State $r, \alpha_0 \assign \checkSat(\rlx{\phi})$ \tikzmark{top1}
  \If{$r = \unsat \lor \alpha_0\models\phi$} {\bf return} $r, \alpha_0$ \tikzmark{bottom1}
  \EndIf
  \State $k,f \assign 0, \initialPhase(\alpha_0)$ \tikzmark{top2}
  \State $\alpha, c \assign \optSat(f, \rlx{\phi})$
\While{$c > 0 \land \neg\allVisited() \land k < T$} 
\State $f' \assign \propose(f)$
\State $\alpha', c' \assign \optSat(f', \rlx{\phi})$
\If {$\accept(c,c')$} $f,c, \alpha \assign f', c', \alpha'$
\Else $\: k \assign k + 1$
\EndIf
\EndWhile
\If {$c = 0$} {\bf return} $\sat{}, \alpha$
\Else {\bf $\:$return} $\allVisited()$ ? $\unsat{}$ :  $\unknown{}$ \tikzmark{right2}
\EndIf\tikzmark{bottom2}
\vspace{-4.5mm}
\EndFunction
\end{algorithmic}
\AddNote{top1}{bottom1}{right2}{\scriptsize Phs.~I}
\AddNote{top2}{bottom2}{right2}{\scriptsize Phs.~II}
\caption{Analyzing $\phi$ with \sys .\label{alg:stoc}}
\end{algorithm}
\end{minipage}
\vspace{-6mm}
\end{wrapfigure}
Our procedure \sys, shown in Algorithm~\ref{alg:stoc}, takes an input
verification query $\phi$ and tries to determine its satisfiability.
\sys follows the standard two-phase convex optimization approach.
Phase I finds \emph{some} assignment $\alpha_0$ satisfying $\rlx{\phi}$,
and phase II attempts to optimize the assignment using the M-H algorithm. Phase
II uses a convex optimization procedure $\optSat$ which takes an objective
function $f$ and a formula $\phi$ as inputs and returns a pair $\alpha, c$,
where $\alpha\models\phi$ and $c = \costOf(f, \phi)$ is the optimal
value of $f$. Phase II chooses an initial phase pattern $f$ 
based on $\alpha_0$ (Line 6) and computes its optimal value $c$. The
M-H algorithm repeatedly proposes a new phase pattern $f'$ (Line 9), computes
its optimal value $c'$, and decides whether to accept $f'$ as the current
phase pattern $f$. The procedure returns \sat{} when a
phase pattern $f$ is found such that $\costOf(f,\rlx{\phi})=0$ and \unsat{} if
all phase patterns have been considered ($\allVisited$ returns true) before a
threshold of $T$ rejections is exceeded. Otherwise, the analysis is inconclusive
(\unknown). 

The $\accept$ method decides whether a proposal is accepted based on the Metropolis ratio
(see Section~\ref{sec:soi}). Function $\initialPhase$ proposes the initial phase pattern $f$ 
induced by the activation pattern corresponding to assignment
$\alpha_0$. Our proposal strategy ($\propose$) is also simple: pick a ReLU $n$
at random and flip its cost component in the current phase pattern $f$
(either from $n^a - n^b$ to $n^a$, or vice-versa). This proposal strategy is
symmetric, ergodic, and performs well in practice. Both the initialization strategy and the
proposal strategy are crucial to the performance of the M-H Algorithm, and
exploring
more sophisticated strategies
is a promising avenue for future work.
Importantly, the same convex procedure is used in both phases.
Therefore, from the perspective of the convex procedure, \sys solves a
sequence of convex optimization problems that differ only in the objective functions,
and each problem can be solved incrementally by updating
the phase pattern without the need for a restart.

\subsection{Complete Analysis and Pseudo-impact Branching}
\label{subsec:pi}

To extend a canonical complete verification procedure (i.e.,
Algorithm~\ref{alg:complete}), its $\checkSat$ call is replaced with \sys. Note
that the implementation of $\branch$ in this algorithm has a significant
influence on its performance. Here, we consider an \soi-aware implementation of
$\branch$, which makes decisions by selecting a particular ReLU to be active or
inactive.  The choice of \emph{which} ReLU is crucial. Intuitively, we want to branch on the ReLU with
the most impact on the value of $\soifunc$. After branching, $\sys$ should be
closer to either: finding a satisfying assignment (if $\soifunc$ is decreased), or
determining unsatisfiability (if $\soifunc$ is increased).
Computing the exact impact of each ReLU $n$ on $\soifunc$ would be expensive; however, we can
estimate it by recording changes in $\soifunc$ during the execution of \sys.

Concretely, for each ReLU $n$, we maintain its \emph{pseudo-impact},\footnote{
The name is in analogy to pseudo-cost branching heuristics in MILP,
where the integer variable with the largest impact on the objective function is chosen~\cite{pseudo-cost}.}
$\pseudo(n)$, which represents the estimated impact of $n$ on $\soifunc$.  For each $n$, $\pseudo(n)$ is initialized to 0. Then during the M-H algorithm, whenever the next proposal flips the cost component
of ReLU $n$, we calculate the local impact on $\soifunc$: $\Delta = \abs{\costOf(f,\rlx{\phi}) - \costOf(f',\rlx{\phi})}$.
We use $\Delta$ to update the value of $\pseudo(n)$ according to the \emph{exponential moving
  average} (EMA):
$\pseudo(n) = \gamma * \pseudo(n) + (1 - \gamma) \cdot \Delta$,
where $\gamma$ attenuates previous estimates of $n$'s impact.
We use EMA because recent estimates are more likely to be relevant to the current phase pattern.
At branching time, the \emph{pseudo-impact heuristic} picks $\argmax_n \pseudo(n)$ as
the ReLU to split on. The heuristic is inaccurate early in the search, so
we use a static branching order
(e.g.,~\cite{wu2020parallelization,babsr}) while the depth of the search tree is 
shallow (e.g., $< 3$).

\section{Experimental Evaluation}
\label{sec:experiments}

In this section, we present an experimental evaluation of the proposed techniques.  Our experiments include:
\begin{inparaenum}
\item an ablation study to examine the effect of each proposed technique;
\item a run-time comparison of our prototype with other complete analyzers;
\item an empirical study of the choice of the rejection threshold $T$ in Algorithm \ref{alg:stoc};
  and
\item an experiment in using our analysis procedure to improve the perturbation bounds found by
AutoAttack~\cite{autoattack}, an adversarial attack algorithm.
\end{inparaenum}

\subsection{Implementation.}

We implemented our techniques in Marabou~\cite{marabou},
an open-source toolbox for analyzing Neural Networks.
It features a user-friendly python API for defining properties
and loading networks, and a native implementation of the Simplex
algorithm.
Besides the Markov chain Monte Carlo stochastic local search
algorithm presented in Section~\ref{subsec:deepsoi} and the pseudo-impact branching
heuristic presented in Section~\ref{subsec:pi}, 
we also implemented a Walksat-inspired~\cite{selman1994noise} stochastic
local search strategy to evaluate the effectiveness
of MCMC-sampling as a local minimization strategy.
Concretely, from a phase pattern
$f$, the strategy greedily moves to a neighbor $f'$ of $f$, with
$\costOf(f',\rlx{\phi}) < \costOf(f,\rlx{\phi})$. If no such $f'$ exists
(i.e., a local minimum has been reached), the strategy moves to a random neighbor.

The $\checkSat$ and $\optSat$ methods in Algorithm~\ref{alg:stoc} can be instantiated with either
the native Simplex engine of Marabou or with Gurobi, an off-the-shelf (MI)LP-solver.
The $\tightenBounds$ method is instantiated with the DeepPoly analysis from~\cite{deeppoly},
an effective and light-weight bound-tightening pass, which is also 
implemented in Marabou.

\subsection{Benchmarks.}
We evaluate on networks from four different applications:
\benchmark{MNIST}, \benchmark{CIFAR10}, \benchmark{TaxiNet}, and \benchmark{GTSR}.
The network architectures are shown in Table \ref{tab:arch}.

\begin{wrapfigure}{r}{0.48\textwidth}
  \vspace{-6mm}
  \begin{minipage}[t]{5cm}
  \setlength\tabcolsep{4pt}
\centering		
\sffamily
\begin{scriptsize}
\begin{tabular}{lllcc}
\toprule
Bench. & Model & Type & ReLUs & Hid.~Layers \\
\cmidrule{1-5}
\benchmark{MNIST} & \benchmark{MNIST_1} & FC & 512 & 2 \\
                  & \benchmark{MNIST_2} & FC & 1024 & 4 \\
                  & \benchmark{MNIST_3} & FC & 1536 & 6 \\
\cmidrule{1-5}
\benchmark{TaxiNet} & \benchmark{Taxi1} & Conv & 688 & 6 \\
                  & \benchmark{Taxi2} & Conv & 2048 & 4 \\
                  & \benchmark{Taxi3} & Conv & 2752 & 6 \\

\cmidrule{1-5}
\benchmark{CIFAR10} & \benchmark{CIFAR10_b} & Conv & 1226 & 4 \\
                  & \benchmark{CIFAR10_w} & Conv &  4804 & 4 \\
                  & \benchmark{CIFAR10_d} & Conv & 5196 & 6 \\
\cmidrule{1-5}
\benchmark{GTSR} & \benchmark{GTSR_1} & FC & 600 & 3 \\
                  & \benchmark{GTSR_2} & Conv & 2784 & 4 \\

\bottomrule
\end{tabular}
\vspace{-2mm}
\caption{Architecture overview.}
\label{tab:arch}
\end{scriptsize}
  \end{minipage}
  \vspace{-8mm}
\end{wrapfigure}

The \benchmark{MNIST}~\cite{mnist} and \benchmark{CIFAR10}~\cite{cifar10}
networks are established benchmarks used in previous literature (e.g., \cite{scaling,peregrinn,wu2020parallelization,xu2020fast}) as well as
in the 2021 VNN Competition~\cite{vnncomp-21}.
Notably, the same \benchmark{MNIST} networks are used to evaluate
the original \peregrinn work.

For \benchmark{MNIST} and \benchmark{CIFAR10} networks, we check
robustness against targeted $l_{\infty}$ attacks on randomly selected images
from the test sets.
The target labels are chosen randomly from the incorrect labels,
and the perturbation bound is sampled uniformly from
$\{0.01, 0.03, 0.06, 0.09, 0.12, 0.15\}$. 
%
 The \benchmark{TaxiNet}~\cite{taxi} benchmark set
comprises robustness queries over regression models used for vision-based
autonomous taxiing. Given an image of the taxiway captured by the
aircraft, the model predicts its displacement (in meters) from the center of the
taxiway. A controller uses the output to adjust the heading of the aircraft.
Robustness is parametrized by input perturbation $\delta$ and output perturbation
$\epsilon$; we sample $(\delta, \epsilon)$ uniformly from $\{0.01,0.03, 0.06\} \times \{2, 6\}$. 
%
The \benchmark{GTSR} benchmark set comprises 
robustness queries on image classifiers trained on a
subset of the German Traffic Sign Recognition benchmark set \cite{gtsr}.
Given a $32\times32$ RGB image the networks classify it as one of
seven different kinds of traffic signs.
A hazing perturbation~\cite{deepcert} drains color from the image to create
a veil of colored mist.
 Given an image $I$, a perturbation parameter $\epsilon$,
and a haze color $C^{f}$, the perturbed image $I'$ is equal to
$(1-\epsilon) \cdot I + \epsilon \cdot C^{f}$.
The robustness queries check whether the bound yielded by the test-based method
in \cite{deepcert} is minimal.
All pixel values are normalized to
$[0,1]$, and the chosen perturbation values yield a mix of non-trivial
\sat{} and \unsat{} instances.


\subsection{Experimental Setup.}

%
%

Experiments are run on a cluster equipped with Intel Xeon E5-2637 v4 CPUs running
Ubuntu 16.04. Unless specified otherwise, each job is run with 1 thread,
8GB memory, and a 1-hour CPU timeout.
By default, the $\checkSat$ and $\optSat$ methods use Gurobi.
The following hyper-parameters are used: the rejection threshold $T$ in Algorithm
\ref{alg:stoc} is 2; the discount factor $\gamma$ in the EMA is 0.5; and the
probability density parameter $\beta$ in the Metropolis ratio is 10. These parameters
are empirically optimal on a subset of MNIST benchmarks.
In practice, the performance
is most sensitive to the rejection threshold $T$, and below
(Section~\ref{subsec:T}), we conduct experiments to
study its effect.


\begin{table}[t]
\setlength\tabcolsep{5pt}
\centering		
\sffamily
\begin{scriptsize}
\begin{tabular}{lccccccaacc}
  \toprule
Bench. (\#) & \multicolumn{2}{c}{\milp}
& \multicolumn{2}{c}{\lpSnc}
& \multicolumn{2}{c}{\soiSnc}
& \multicolumn{2}{c}{\soiMcmc}
& \multicolumn{2}{c}{\soiWSat}\\
\cmidrule(lr){2-3} \cmidrule(lr){4-5} \cmidrule(lr){6-7} \cmidrule(lr){8-9} \cmidrule(lr){10-11}
& Solv. & Time & Solv. & Time & Solv. & Time & Solv. & Time & Solv. & Time \\
\cmidrule{1-11}
\benchmark{MNIST_1}(90) & \textbf{77} & 19791 & 47 & 6892 & 66 & 5635 & 70 & 5976 & 68 & 5388  \\
\benchmark{MNIST_2}(90) & 29 & 6125 & 24 & 514 & \textbf{36} & 4356 & 31 & 757 & 31 & 909 \\
\benchmark{MNIST_3}(90) & 23 & 957 & 21 & 1609 & 34 & 9519 & \textbf{35} & 8327 & 33 & 5270 \\
\cmidrule{1-11}
\benchmark{Taxi_1}(90) & \textbf{90} & 786 & 61 & 9054 & 80 & 4257 & 89 & 1390 & 90 & 1489 \\
\benchmark{Taxi_2}(90) & 40 & 17093 & 2 & 891 & 70 & 5503 & \textbf{71} & 6889 & 71 & 7407 \\
\benchmark{Taxi_3}(90) & \textbf{89} & 5058 & 64 & 69715 & 87 & 1034 & 88 & 2164 & 87 & 997 \\
\cmidrule{1-11}
\benchmark{CIFAR10_b}(90) & \textbf{76} & 4316 & 26 & 7425 & 69 & 6286 & 73 & 16469 & 69 & 5200 \\
\benchmark{CIFAR10_w}(90) & 38 & 9879 & 18 & 845 & 41 & 4619 & 42 & 8129 & \textbf{42} & 6415 \\
\benchmark{CIFAR10_d}(90) & 30 & 4198 & 21 & 3395 & 51 & 17679 & 51 & 15056 & 51 & 15015 \\
\cmidrule{1-11}
\benchmark{GTSR_1}(90) & 90 & 2541 & \textbf{90} & 2435 & 89 & 4900 & 90 & 15238 & 90 & 4805 \\
\benchmark{GTSR_2}(90) & 90 & 23613 & \textbf{90} & 4456 & 90 & 7507 & 90 & 10426 & 90 & 6180 \\
\cmidrule{1-11}
Total (990) & 673 & 94354 & 463 & 107230 & 711 & 71294 & \textbf{730} & 90822 & 721 & 59073 \\
\bottomrule
\end{tabular}
\vspace{1mm}
\caption{Instances solved by different configurations and their runtime (in seconds) on
  solved instances. \label{tab:ablation}}
\end{scriptsize}
\vspace{-1cm}
\end{table}

\subsection{Ablation study of the proposed techniques.}

To evaluate each individual component of our proposed techniques, we run several
configurations across the full set of benchmarks described above.

We first consider two baselines that do not minimize the \soi:
\begin{inparaenum}
\item \lpSnc --- runs Algorithm~\ref{alg:complete} with the Split-and-Conquer~(SnC)
  branching heuristic~\cite{wu2020parallelization}, which estimates the number of
  tightened bounds from a ReLU split;
\item \milp --- encodes the query in Gurobi using MIPVerify's MILP
  encoding~\cite{mipverify}.%
\footnote{ This configuration does not use the LP/MILP-based preprocessing
    passes from MIPVerify \cite{mipverify} because they degrade performance on our benchmarks.}
\end{inparaenum}

We then evaluate three configurations of \soi-based complete analysis
parameterized by the branching heuristic and the \soi-minimization
algorithm:
\begin{inparaenum}
\item \soiSnc --- runs \sys with the SnC branching heuristic;
\item \soiMcmc --- runs \sys with the pseudo-impact~(PI) heuristic;
\item \soiWSat --- runs the Walksat-based algorithm with the PI heuristic.  
\end{inparaenum}
Each \soi configuration differs in one parameter w.r.t. the previous, so that
pair-wise comparison highlights the effect of that parameter.

Table~\ref{tab:ablation} summarizes the runtime performance of different configurations
on the four benchmark sets. The three configurations that minimize
the \soi, namely \soiMcmc, \soiWSat and \soiSnc,
all solve significantly more instances than the two baseline configurations.
In particular, \soiSnc solves 248 (53.4\%) more instances than \lpSnc. Since
all configurations start with the same variable bounds derived by the
DeepPoly analysis, the performance gain is mainly due to the use of \soi.

Among the three \soi configurations, the one with both \solver{pi} and \solver{mcmc} solves
the most instances. In particular, it solves 8 more instances than \soiWSat,
suggesting that MCMC sampling is, overall, a better approach
than the Walksat-based strategy. On the other hand,  \soiMcmc and \soiSnc show
complementary behaviors. For instance, the latter solves 5 more instances on
\benchmark{MNIST_1}, and the former solves 11 more on the \benchmark{Taxi} benchmarks. 
This motivates a portfolio configuration \portfolio,which runs \soiMcmc and \soiSnc
in parallel. This strategy is able to solve 742 instances overall with a 1-hour
wall-clock timeout, yielding a gain of at least 12 more solved instances
compared with any single-threaded configuration.

\subsection{Comparison with other complete analyzers.}

In this section, we compare our implementation with other complete analyzers.
We first compare with \peregrinn, which as described in Section~\ref{sec:related}
introduces a heuristic cost function to guide the search.
We evaluate \peregrinn on the \benchmark{MNIST} networks, the same set of networks
used in its original evaluation. We did not run \peregrinn on the other benchmarks
because it only supports \texttt{.nnet} format, which is designed for fully connected feed-forward
ReLU networks.

\begin{table}[t]
\setlength\tabcolsep{8pt}
\centering		
\sffamily
\begin{scriptsize}
\begin{tabular}{laacccccc}
  \toprule
  Bench. (\#)
  & \multicolumn{2}{c}{\soiMcmc}
  & \multicolumn{2}{c}{\solver{PeregriNN}}
  & \multicolumn{2}{c}{\solver{ERAN_{1}}}
  & \multicolumn{2}{c}{\solver{ERAN_{2}}} \\
  \cmidrule(lr){2-3} \cmidrule(lr){4-5} \cmidrule(lr){6-7} \cmidrule(lr){8-9}
& Solv. & Time & Solv. & Time & Solv. & Time & Solv. & Time \\
\cmidrule{1-9}
\benchmark{MNIST_1}(90) & 70 & 5976 & 64 & 11117 & \textbf{76} & 18679 & 75 & 19520 \\
\benchmark{MNIST_2}(90) & \textbf{31} & 757 & 31 & 2287 & 28 & 1910 & 28 & 3126 \\
\benchmark{MNIST_3}(90) & \textbf{35} & 8327  & 26 & 2344 & 24 & 1538 & 24 & 3292 \\
\cmidrule{1-9}
\benchmark{Taxi_1}(90) & 89 & 1390 & - & -  & \textbf{90} & 1653 & 90 & 3262 \\
\benchmark{Taxi_2}(90) & \textbf{71} & 6889 & - & -  & 40 & 16460 & 35 & 31778 \\
\benchmark{Taxi_3}(90) & 88 & 2164 & - & -  & \textbf{88} & 1389 & 88 & 4581 \\
\cmidrule{1-9}
\benchmark{CIFAR10_b}(90) & 73 & 16469 & - & - & \textbf{77} & 4604 & 77 & 14269 \\
\benchmark{CIFAR10_w}(90) & \textbf{42} & 8129 & - & - & 41 & 14403 & 37 & 14453 \\
\benchmark{CIFAR10_d}(90) & \textbf{51} & 15056 & - & - & 31 & 7587 & 26 & 5245 \\
\cmidrule{1-9}
\benchmark{GTSR_1}(90) & 90 & 15238 & - & - & \textbf{90} & 2023 & 90 & 32585 \\
\benchmark{GTSR_2}(90) & \textbf{90} & 10426 & - & - & 78 & 77829 & 75 & 81232 \\
\cmidrule{1-9}
Total (990) & \textbf{730} & 90822 & - & - & 663 & 148075 & 645 & 213343 \\
\bottomrule
\end{tabular}
\vspace{1mm}
\caption{Instances solved by different complete verifiers and their runtime (in seconds)
  on solved instances. \label{tab:other}}
\end{scriptsize}
\vspace{-1cm}
\end{table}

In addition, we also compare with ERAN, a state-of-the-art complete analyzer based on
abstract interpretation, on the full set of benchmarks.
ERAN is often used as a strong baseline in recent neural network verification
literature and was among the top performers in the past VNN Competition 2021.
We compare with two ERAN configurations:
\begin{inparaenum}
\item \solver{ERAN_1} --- ERAN using the DeepPoly analysis~\cite{deeppoly}
  for abstract interpretation and Gurobi for solving;
\item \solver{ERAN_2} --- same as above except using the k-ReLU analysis~\cite{kpoly} for
  abstract interpretation.
\end{inparaenum}
We choose to compare with ERAN instead of other state-of-the-art
neural network analyzers, e.g., alpha-beta crown~\cite{crown,bcrown}, OVAL~\cite{scaling},
and fast-and-complete~\cite{xu2020fast}, mainly because
the latter tools are GPU-based, while ERAN supports execution on
CPU, where our prototype is designed to run. This makes a fair comparison possible.
Note that our goal in this section is not to claim superiority over
all state-of-the-art solvers. Rather, the goal is to provide assurance that our
implementation is reasonable. As explained earlier, our approach can be integrated
into other complete search shells with different search heuristics, and is orthogonal to
techniques such as GPU-acceleration, parallelization, and tighter convex
relaxation (e.g., beyond the Planet relaxation), which are all future
development directions for Marabou.

Table \ref{tab:other} summarizes the runtime performance of different solvers.
We include again our best configuration, \soiMcmc, for ease of comparison.
On the three \benchmark{MNIST} benchmark sets, \peregrinn either solves fewer instances than
\soiMcmc or takes longer time to solve the same number of instances.
We note that \peregrinn's heuristic objective function could be employed during the
feasibility check of \sys (Line 4, Algorithm~\ref{alg:stoc}). Exploring this
complementarity between \peregrinn and our approach  is left as future work.

\begin{wrapfigure}{r}{0.5\textwidth}
  \vspace{-1cm}
  \begin{center}
    \includegraphics[width=0.5\textwidth]{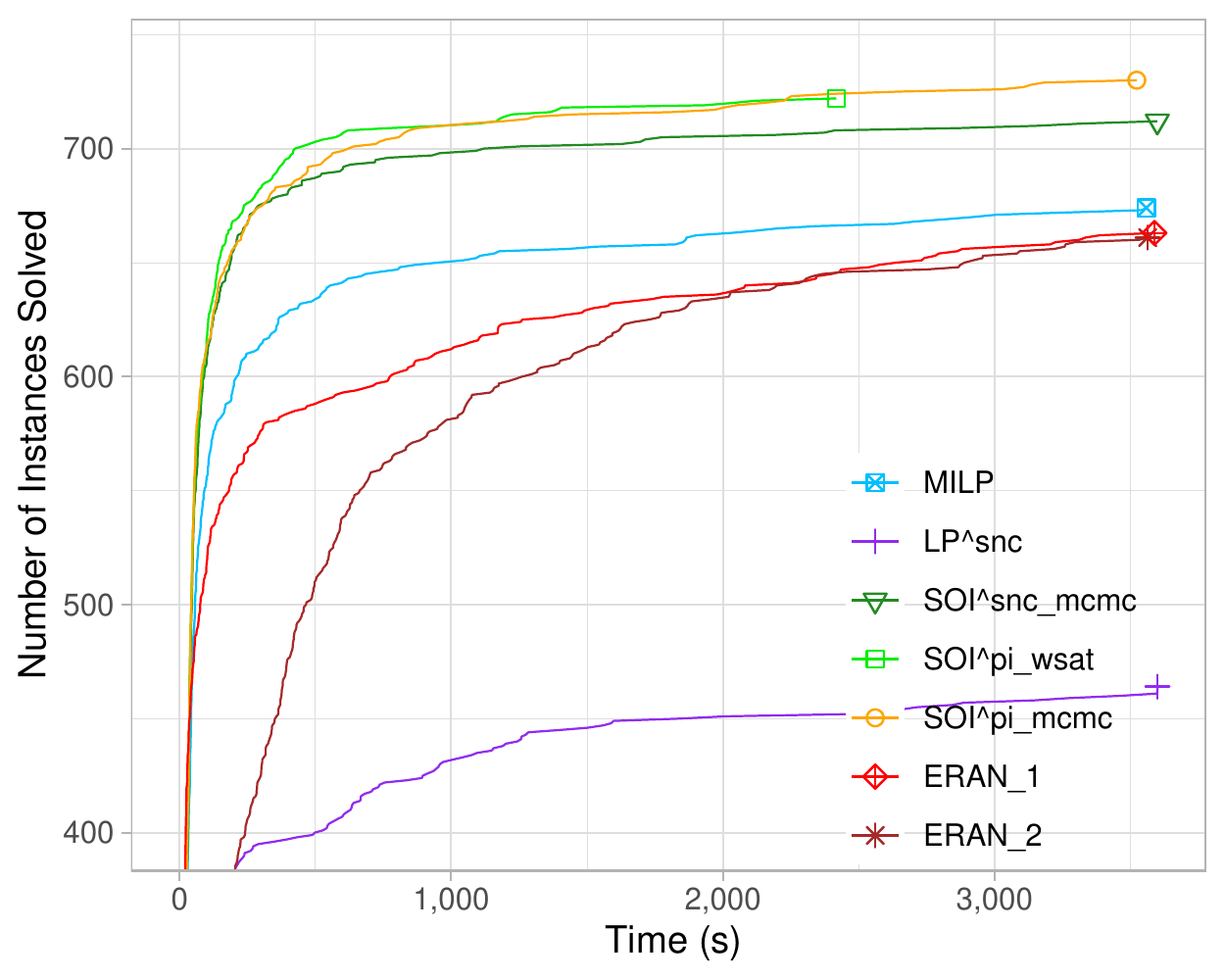}
  \end{center}
  \vspace{-5mm}
    \caption{Cactus plot on all benchmarks. \label{fig:cactus}}
    \vspace{-6mm}
\end{wrapfigure}
Compared with \solver{ERAN_1} and \solver{ERAN_2}, \soiMcmc also solves significantly
more instances overall, with a performance gain of at least
10.1\% more solved instances. Taking a closer look at the performance breakdown
on individual benchmarks, we observe complementary behaviors between
\soiMcmc and \solver{ERAN_1}, with the latter solving more instances than \soiMcmc
on 3 of the 11 benchmark sets. 
Figure~\ref{fig:cactus} shows the cactus plot of configurations that run on
all benchmarks. \solver{ERAN_1} is able to solve more instances than all the
other configurations when the time limit is short, but is overtaken by the
three \soi-based configurations once the time limit exceeds 30s. One explanation
for this is that the \soi-enabled configurations spend more time probing at each search state,
and for easier instances, it might be more beneficial to branch eagerly.

Finally, we compare the portfolio strategy \portfolio described in the previous
subsection to \solver{ERAN_1} running 2 threads. The latter solves 10.3\% fewer
instances (673 overall). Figure~\ref{fig:scatter} shows a scatter plot of
the runtime performance of these two configurations.
For unsatisfiable instances, most can be resolved efficiently by both solvers, and
each solver has a few unique solves. On the other hand, \portfolio is able to
solve significantly more satisfiable benchmarks.

\begin{figure}[t]
  \begin{minipage}{0.53\textwidth}
    \vspace{-4mm}
    \hspace{-6mm}
    \centering
    \includegraphics[width=5.9cm]{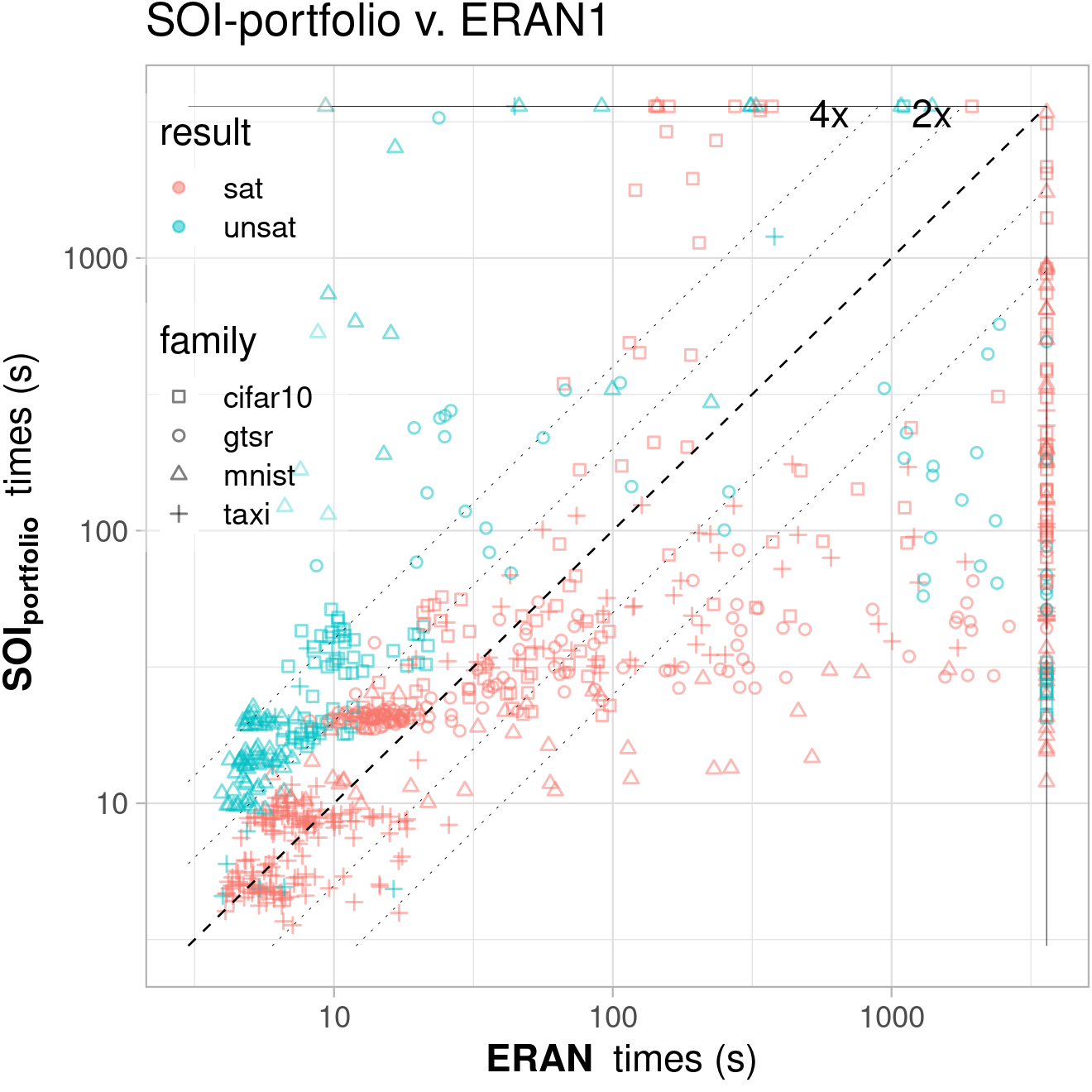}
    \caption{Runtime of \portfolio and \solver{ERAN_1} \\running with 2 threads. \label{fig:scatter}}
  \end{minipage}%
  \begin{minipage}{0.47\textwidth}
    \hspace{1mm}
    \centering
    \includegraphics[width=5.9cm]{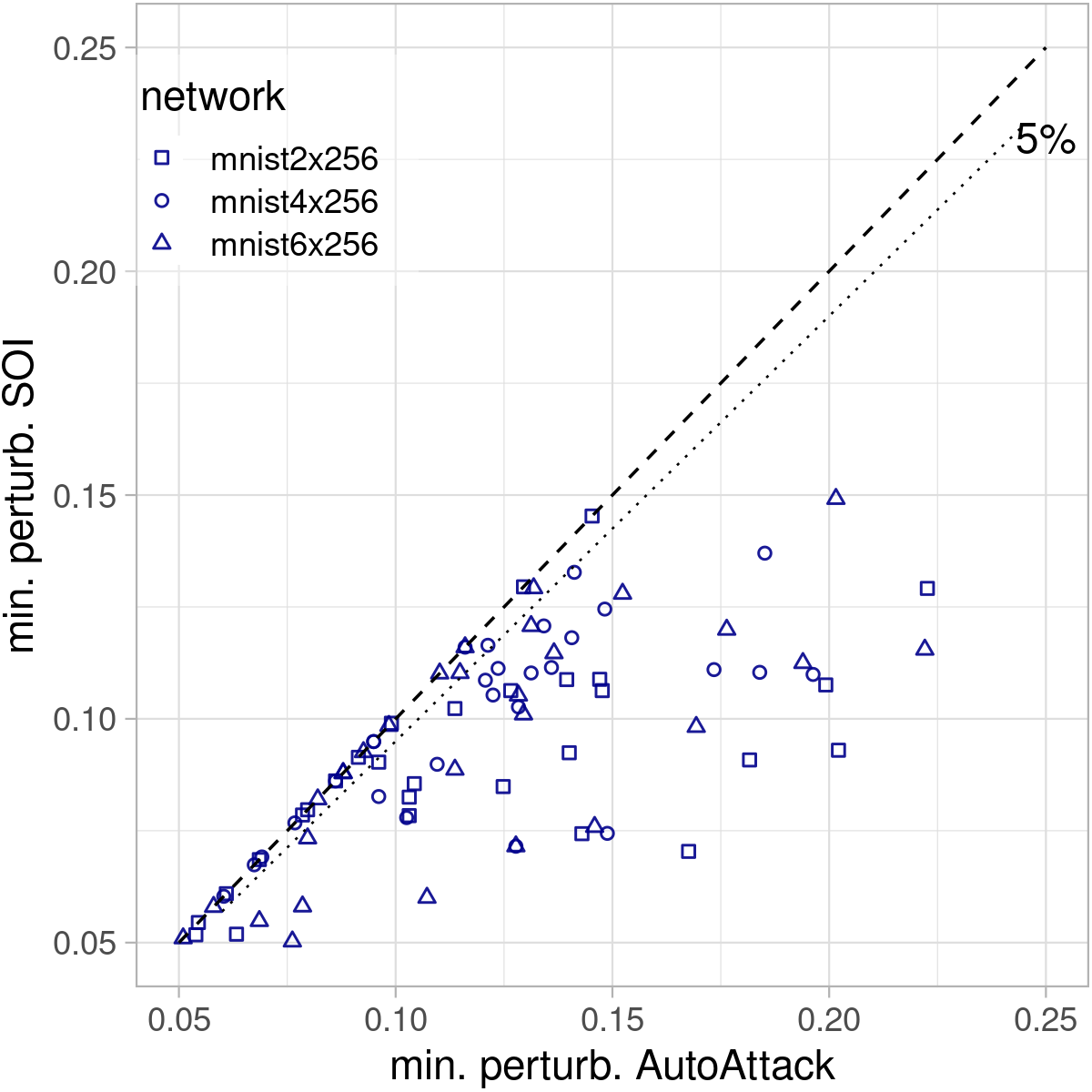}
    \caption{Improvements over the perturbation bounds found by AutoAttack. \label{fig:cw}}
  \end{minipage}
  \label{fig:main}
  \vspace{-4mm}
\end{figure}

\subsection{Incremental Solving and the Rejection Threshold $T$ \label{subsec:T}}
The rejection threshold $T$ in Algorithm \ref{alg:stoc} controls the number of
rejected proposals allowed before returning \unknown.  An \emph{incremental} solver is
one that can accept a sequence of queries, accumulating and reusing relevant
bounds derived by each query. To investigate the interplay of $T$ and incrementality, 
we perform an experiment using the incremental simplex engine in Marabou
while varying the value of $T$. We additionally control the branching order (by
using a fixed topological order). We conduct the experiment on 180 \benchmark{MNIST_1}
and 180 \benchmark{Taxi_1} benchmarks from the aforementioned distributions.
%

Table \ref{tab:t} shows the number of solved instances,
as well as the average time (in seconds) and number of search states on the 95 commonly
solved \unsat{} instances. As $T$ increases, more satisfiable
benchmarks are solved.

\begin{table}[h]
\vspace{-3.5mm}
\setlength\tabcolsep{7pt}
\centering		
\sffamily
\begin{scriptsize}
\begin{tabular}{lcccccc}
\toprule
Rejection threshold $T$ & 1 & 2 & 3 & 4 & 5 & 6 \\
\cmidrule{1-7}
SAT Solv. & 192 & 199 & 196 & 204 & 203 & \textbf{207} \\
\cmidrule{1-7}
UNSAT Solv. & \textbf{91} & 90 & 90 & 89 & 90 & 89 \\
Avg. time (common) & 97.75 & 129.0 & \textbf{83.6} & 108.1 & 137.0 & 187.8 \\
Avg. states (common) & 12948 & 12712 & 6122 & \textbf{5586} & 6404 & 8948 \\
\bottomrule
\end{tabular}
\vspace{2mm}
\caption{Effect of the rejection threshold. \label{tab:t}}
\end{scriptsize}
\vspace{-1cm}
\end{table}
Increasing $T$ can also result in improvement on unsatisfiable
instances---either the average time decreases, or fewer search states are required to solve
the same instance. We believe this improvement is due to the reuse of bounds
derived during the execution of \sys. This suggests that adding incrementality to
the convex solver (like Gurobi) could be highly beneficial for verification applications.
It also suggests that the bounds derived during the simplex execution cannot be subsumed
by bound-tightening analyses such as DeepPoly.

\subsection{Improving the perturbation bounds found by AutoAttack}
Our proposed techniques result in significant performance gain on satisfiable instances.
It is natural to ask whether the satisfiable instances solvable by the \soi-enabled
analysis can also be easily handled by adversarial attack algorithms, which as
mentioned in Section~\ref{sec:related}, focus solely on finding satisfying assignments.
In this section, we show that this is not the case by presenting an experiment where we
use our procedure in combination with AutoAttack~\cite{autoattack}, a state-of-the-art adversarial attack
algorithm, to find higher-quality adversarial examples.

Conceretely, we first use AutoAttack to find an upper bound on the minimal perturbation
required for a successful $l_{\infty}$ attack.
We then use our procedure to search for smaller perturbation bounds, repeatedly
decreasing the bound by 2\% until either UNSAT is proven or a timeout (30 minutes)
is reached. We use the adversarial label of the last successful attack found by
AutoAttack as the target label. We do this for the first 40 correctly classified
test images for the three MNIST architectures, which yields 120 instances.
Figure \ref{fig:cw} shows the improvement of the perturbation bounds.
Reduction of the bound is obtained for 53.3\% of
the instances, with an average reduction of 26.3\%, a median reduction of 22\%,
and a maximum reduction of 58\%. This suggests that our procedure can help obtain a
more precise robustness estimation.

\section{Conclusions and Future Work}
\label{sec:concl}

In this paper, we introduced a
procedure, \sys, for efficiently minimizing the sum of infeasibilities in activation
function constraints with respect to the convex relaxation of a neural network
verification query. We showed how \sys can be integrated into a complete
verification procedure, and we introduced 
a novel \soi-enabled branching heuristic. Extensive experimental results suggest
that our approach is a useful contribution towards scalable analysis of
neural networks.
Our work also opens up multiple promising future directions, including:
1) improving the scalability of \sys by using heuristically chosen subsets
of activation functions in the cost function instead of all unfixed activation
functions; 
2) leveraging parallelism by using GPU-friendly convex procedures or minimizing
the SoI in a distributed manner; 
3) devising more sophisticated initialization and proposal strategies for the
Metropolis-Hastings algorithm;
4) understanding the effects of the proposed branching heuristics on different
types of benchmarks;
5) investigating the use of \sys as a stand-alone adversarial attack algorithm.

\paragraph{Acknowledgements}
We thank Gagandeep Singh for providing useful feedback and Haitham Khedr for
help with running \peregrinn.
This work was partially supported by DARPA (grant FA8750-18-C-0099),
a Ford Alliance Project (199909), NSF (grant
1814369), and the US-Israel Binational Science Foundation (grant 2020250).

\newpage
{
\bibliographystyle{splncs04}
\bibliography{bibli}
}

\newpage
\appendix

\title{Appendix}

\author{
}

\institute{
}

\maketitle

\section{Additional Details on Experimental Setups \label{ap:details}}

\subsection{Additional Details on Benchmarks}
The \benchmark{MNIST} classifiers that we used can be found at the following link:
\url{https://github.com/stanleybak/vnncomp2021/tree/main/benchmarks/mnistfc},
with \benchmark{MNIST_1} corresponding to \benchmark{256x2},
\benchmark{MNIST_2} to \benchmark{256x4},
\benchmark{MNIST_3} to \benchmark{256x6}.
The \benchmark{CIFAR10} classifiers can be found at the following link:
\url{https://github.com/stanleybak/vnncomp2021/tree/main/benchmarks/oval21/nets},
with \benchmark{CIFAR10_b} corresponding to \benchmark{cifar\_base},
\benchmark{CIFAR10_w} to \benchmark{cifar\_wide},
\benchmark{CIFAR10_d} to \benchmark{cifar\_deep}.
The \benchmark{TaxiNet} benchmarks come from \cite{taxi}. 

The \benchmark{GTSR} classifiers and the hazing property come from
DeepCert~\cite{deepcert}.\footnote{\url{https://github.com/DeepCert/contextual-robustness}}
The original paper trains classifiers with three different architectures.
The \benchmark{GTSR_1} model in Table~\ref{tab:arch} is \benchmark{model2a}
in the original work. We did not include the first and the third architectures
in our evaluation. The first architecture contains only 70 ReLUs and the
resulting benchmarks are not challenging enough to differentiate different configurations.
The third architecture contains Max Pooling layers but this paper focuses on ReLUs.
Instead, we trained a second classifier \benchmark{GTSR_2} using the scripts
provided by the authors of \cite{deepcert} with the following architecture:
\begin{table}[h]
\setlength\tabcolsep{7pt}
\centering		
\sffamily
\begin{scriptsize}
\begin{tabular}{cccc}
\toprule
Type & Parameters/Shape & Activation & Dropout \\
\cmidrule{1-4}
Input & $32\times32\times3$ & - & - \\
\cmidrule{1-4}
Conv & $32$ $7\times 7$ filters, strides $3$ & ReLU & None \\
\cmidrule{1-4}
Conv & $32$ $5\times 5$ filters, strides $3$ & ReLU & None\\
\cmidrule{1-4}
Dense & $128\times 1$ & ReLU & $0.5$ \\
\cmidrule{1-4}
Dense & $64\times 1$ & ReLU & $0.5$\\
\cmidrule{1-4}
Dense & $7\times 1$ & None & None\\
\bottomrule
\end{tabular}
\end{scriptsize}
\end{table}

Following the original paper, we used 11220 images for training,
4110 images for validation, and 4110 images for testing.
Adam optimizer with learning rate 0.001 and the
Cross-entropy loss were used for training.

We checked the networks' robustness against
the hazing perturbation~\cite{deepcert},
which drains color from the image to create
a veil of colored mist (Figure~\ref{fig:hazing}).
 Given an image $I$, a perturbation parameter $\epsilon$,
and a haze color $C^{f}$, the perturbed image $I'$ is equal to
$(1-\epsilon) \cdot I + \epsilon \cdot C^{f}$.
The concrete property we verified is whether the
robustness bound discovered by the test-based method
in \cite{deepcert} is minimal. More precisely, if the smallest misclassifying
$\epsilon$ value found for a given image is $\epsilon_0$, we checked whether that
image can be misclassified when $\epsilon\in [0, 0.99\cdot \epsilon_0]$. We
found that this range of $\epsilon$ yields a mixture of interesting satisfiable
and unsatisfiable instances, and choosing smaller upper bounds of $\epsilon$
tends to yield easily unsatisfiable instances. This suggests that the test-based
method in ~\cite{deepcert} can often find high-quality adversarial examples close to the
perturbation bound. We hypothesize that this is due to the low-dimension of the
perturbation in the case of hazing.

\begin{figure}[h!]
\centering
\includegraphics[width=10cm]{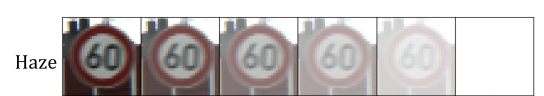}
\caption{Examples (from ~\cite{deepcert}) of applying the hazing
perturbations to \benchmark{GTSR} data.}
\label{fig:hazing}
\end{figure}

\subsection{Licenses}

The open source tools that we used in the paper include \solver{ERAN} (Apache
License 2.0), \solver{Marabou} (3-clause BSD license), and \solver{PeregriNN} (MIT License).
Our experiments also rely on the \texttt{Gurobi Optimizer 9.1},\footnote{\url{https://www.gurobi.com/}}
which is a commercial optimization tool but supports free usage with an academic
license for users within academic institutions.

\subsection{Other Details}
\paragraph{Complete Analysis}
We set the feasibility tolerance of linear constraints (or equivalence) to
$1e^{-9}$ and the integrality restriction tolerance (or equivalence) to
$1e^{-8}$ for all solvers.

For each verification instance, each configuration in Table~\ref{tab:ablation}
either terminated with an answer or timed out. No memory-out, errors, or
inconsistent results were observed.

For PeregriNN, we turn off its random sampling scheme in the beginning for fair
comparison of the decision procedure. The random sampling scheme could be used
as a pre-processing pass by each solver.

\paragraph{Improving the perturbation bounds found by AutoAttack}
AutoAttack~\cite{autoattack} is a state-of-the-art portfolio adversarial attack procedure. It
is designed to generate adversarial examples with respect to $L_1$, $L_2$, or $L_{\infty}$
norm bounded perturbations by iteratively applying four different adversarial attack algorithms.
We used the original AutoAttack library \footnote{\url{https://github.com/fra31/auto-attack}}
and the default parameters.

Since AutoAttack takes the perturbation bound as an input and searches
for an adversarial example within that bound, to adapt it for our use (i.e., minimizing the
perturbation bounds), we wrap AutoAttack in a canonical binary search to minimize
the perturbation bound found by AutoAttack. For the binary search, the initial
upper- and lower- perturbation bounds are set to 0.3 and 0, and the search terminates
when the difference between the upper and the lower bounds is below 0.001. We are able to
obtain at least one adversarial example (therefore a sound perturbation upper-bound)
for each test image this way.

We have also conducted the same experiment with the Carlini-Wagner (CW) Attack~\cite{cw}, which
does not take the perturbation bound as an input but instead tries to find a minimal perturbation.
We are able to reduce the perturbation bounds found by the CW attack for 85.2\% of the instances,
with a median reduction of 6\% and a maximum reduction of 14\%. We use the result on AutoAttack
in our paper because it is a more recent attack.

\section{Additional Experimental Results \label{ap:results}}

\subsection{Comparison with Marabou}

In addition to the baselinse and solvers described in the paper, we also compared with
the Reluplex~\cite{reluplex} procedure implemented in Marabou.
The results are shown in Table~\ref{tab:additional}. For reference, we also include
the run-time performance of \soiMcmc, our main configuration.
\begin{table}[h!]
\vspace*{1mm}
\setlength\tabcolsep{6pt}
\centering		
\sffamily
\begin{scriptsize}
\begin{tabular}{lcccc}
\toprule
Benchmark & \multicolumn{2}{c}{\solver{Reluplex}}
& \multicolumn{2}{c}{\soiMcmc} \\
\cmidrule(lr){2-3} \cmidrule(lr){4-5}
& Solv. & Time & Solv. & Time \\
\cmidrule{1-5}
\benchmark{MNIST1}(90) & 36 & 8429 & 70 & 5784 \\
\benchmark{MNIST2}(90) & 23 & 542 & 31 & 760 \\
\benchmark{MNIST3}(90) & 20 & 4357 & 35 & 7883 \\
\cmidrule{1-5}
\benchmark{Taxi1}(90) & 53 & 3589 & 89 & 2172 \\
\benchmark{Taxi2}(90) & 0 & 0 & 71 & 6369 \\
\benchmark{Taxi3}(90) & 0 & 0 & 87 & 1024 \\
\cmidrule{1-5}
\benchmark{CIFAR10_b}(90) & 33 &5493 & 74 & 19310 \\
\benchmark{CIFAR10_w}(90) & 0 & 0 & 42 & 8187 \\
\benchmark{CIFAR10_d}(90) & 6 & 119 & 51 & 14643 \\
\cmidrule{1-5}
\benchmark{GTSR_1}(90) & 90 & 62590 & 90 & 13456 \\
\benchmark{GTSR_2}(90) & 2 & 510 & 90 & 9275  \\
\cmidrule{1-5}
Total (990) & 263 & 85629 & 730 & 88863 \\
\bottomrule
\end{tabular}
\vspace{1mm}
\caption{\# instances solved by The Reluplex procedure and
its run-time on solved instances.
\label{tab:additional}}
\end{scriptsize}
\end{table}

\subsection{Results on MNIST and CIFAR10 with an alternative benchmarking scheme}

We have also considered a different benchmarking scheme for \benchmark{MNIST}
and \benchmark{CIFAR10}.
Instead of choosing the target labels randomly, we always pick the ``runner-up''
class (the label with the highest score other than the correct label).
We sample correctly classified test images from the same distribution of perturbation
bounds mentioned in Section~\ref{sec:experiments}.
The results are shown in Table~\ref{tab:main-alt}. The
performance pattern is similar to that in Table~\ref{tab:ablation}, where the
configurations that minimize the \soi (the last three configurations)
solve significantly more instances than the other configurations.

\begin{table}[h!]
\setlength\tabcolsep{5pt}
\centering		
\sffamily
\begin{scriptsize}
\begin{tabular}{lccccccccaa}
  \toprule
Bench. (\#) & \multicolumn{2}{c}{\milp}
& \multicolumn{2}{c}{\lpSnc}
& \multicolumn{2}{c}{\soiSnc}
& \multicolumn{2}{c}{\soiWSat}
& \multicolumn{2}{c}{\soiMcmc} \\
\cmidrule(lr){2-3} \cmidrule(lr){4-5} \cmidrule(lr){6-7} \cmidrule(lr){8-9} \cmidrule(lr){10-11}
& Solv. & Time & Solv. & Time & Solv. & Time & Solv. & Time & Solv. & Time \\
\cmidrule{1-11}
\benchmark{MNIST_1}(90) & 80 & 8207 & 42 & 7187 & 68 & 1578 & 74 & 11573 & \textbf{77} & 8124 \\
\benchmark{MNIST_2}(90) & 33 & 1280 & 27 & 945 & \textbf{51} & 1843 & 49 & 5724 & 50 & 6793 \\
\benchmark{MNIST_3}(90) & 23 & 2394 & 24 & 8581 & 33 & 7174 & 31 & 1444 & \textbf{33} & 2314 \\
\cmidrule{1-11}
\benchmark{CIFAR10_b}(90) & \textbf{65} & 10581 & 16 & 3234 & 56 & 6510 & 55 & 3135 & 61 & 20034 \\
\benchmark{CIFAR10_w}(90) & 36 & 8257 & 18 & 870 & 37 & 9589 & 36 & 7386 & \textbf{37} & 8105 \\ 
\benchmark{CIFAR10_d}(90) & 33 & 6825 & 26 & 890 & \textbf{49} & 10673 & 48 & 8950 & 49 & 12046 \\
\cmidrule{1-11}
Total(540) & 270 & 37544 & 153 & 21707 & 294 & 37367 & 293 & 38212 & \textbf{307} & 57416 \\
\bottomrule
\end{tabular}
\vspace{1mm}
\caption{\# instances solved by different configurations and their run-time on
  solved instances under an alternative benchmarking scheme. \label{tab:main-alt}}
\end{scriptsize}
\vspace{-0.7cm}
\end{table}

\section{Probability of Sampling an Arbitrary Phase Pattern \label{ap:proof}}
\begin{theorem}
Given a neural network verification query $\phi$, its corresponding \soi
function $\soifunc$, and the set of phase patterns $\soiset$, if
we use the Metropolis ratio as the acceptance probability and 
the proposal strategy described in Section~\ref{sec:complete},
the probability that the M-H algorithm \emph{ever} samples an arbitrary phase
pattern $s\in\soiset$ converges to 1 as the length of the constructed Markov
chain goes to $\infty$.
\end{theorem}
\begin{proof}
Recall that the proposal strategy in Section~\ref{sec:complete} is to randomly
select a ReLU and flip its error term in the current phase pattern $f$.
This means that for any arbitrary pair of phase patterns $f_1, f_2\in\soiset$,
there exists a finite sequence of proposals that changes $f_1$ to
 $f_2$. Moreover, since the acceptance probability ensures that there is some
 probability to move from $f_1$ to any of its neighbors, the probability of going from
 $f_1$ to $f_2$  in some finite steps is positive.
 Since the state space is finite, the expected number of steps to reach $f_2$
 from $f_1$ is finite. Therefore, as the number of steps goes to $\infty$,
 the probability that the Markov chain contains $f_2$ goes to 1.
\end{proof}

\vfill

{\small\medskip\noindent{\bf Open Access} This chapter is licensed under the terms of the Creative Commons\break Attribution 4.0 International License (\url{http://creativecommons.org/licenses/by/4.0/}), which permits use, sharing, adaptation, distribution and reproduction in any medium or format, as long as you give appropriate credit to the original author(s) and the source, provide a link to the Creative Commons license and indicate if changes were made.}

{\small \spaceskip .28em plus .1em minus .1em The images or other third party material in this chapter are included in the chapter's Creative Commons license, unless indicated otherwise in a credit line to the material.~If material is not included in the chapter's Creative Commons license and your intended\break use is not permitted by statutory regulation or exceeds the permitted use, you will need to obtain permission directly from the copyright holder.}

\medskip\noindent\includegraphics{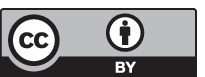}

\end{document}